\newcommand{\titt}{Improved SVRG for quadratic functions}
\newcommand{\commentt}[2]{#2}
\newcommand{\comment}[1]{}
\newcommand{\cov}{{\rm rCov}}
\newcommand{\var}{{\rm Var}}

\newcommand{\tr}{{\rm tr}}

\newcommand{\citet}{\citeasnoun}
\documentclass[a4paper,10pt]{article}
\usepackage{slashbox,amsthm}
\usepackage[full]{harvard}

\usepackage[ps,dvips]{xy}
\title{\titt}
\author{Nabil Kahal\'e
\thanks{\emph{ESCP Business School, 75011 Paris,
France; {e-mail: }{nkahale@escp.eu. ORCID: 0000-0002-4009-6815.} }}}
\date{\today}
\usepackage{amstext}
\usepackage{setspace}
\usepackage[margin=1in]{geometry}
\usepackage[english]{babel}
\usepackage{color}
\usepackage{textcomp}
\usepackage{varioref}
\usepackage{algpseudocode}
\usepackage{algorithm}
\usepackage{xspace}
\usepackage{tikz,pgfplots}
\usepackage{caption,subcaption}
\usepackage{amsmath, amssymb, graphicx, url}

\begin{document}

\newtheorem{example}{Example}[section]
\newtheorem{theorem}{Theorem}[section]
\newtheorem{conjecture}{Conjecture}[section]
\newtheorem{lemma}{Lemma}[section]
\newtheorem{proposition}{Proposition}[section]
\newtheorem{remark}{Remark}[section]
\newtheorem{corollary}{Corollary}[section]
\newtheorem{definition}{Definition}[section]
\numberwithin{equation}{section}
%\numberwithin{proposition}{section}
\maketitle\newcommand{\ABSTRACT}[1]{\begin{abstract}#1\end{abstract}}
\newcommand{\citep}{\cite}

\ABSTRACT{
We analyse an iterative algorithm to minimize quadratic functions whose Hessian matrix $H$ is the expectation of a random symmetric  $d\times d$ matrix. The algorithm is a variant of  the stochastic variance reduced gradient (SVRG). In several applications, including least-squares regressions, ridge regressions, linear discriminant analysis and regularized linear discriminant analysis, the running time of each iteration is proportional to $d$. Under smoothness and convexity conditions, the algorithm has linear convergence. When applied to quadratic functions, our analysis improves the state-of-the-art performance of SVRG up to a logarithmic factor. Furthermore, for well-conditioned quadratic problems, our analysis improves the state-of-the-art running times of accelerated SVRG,  and is better than the known  matching lower bound, by a logarithmic factor. Our theoretical results are backed with numerical experiments.
}
\section{Introduction}
The recent availability of massive volumes of data fosters the need to design   computationally efficient algorithms for optimization in high dimensions. In large-scale machine learning, stochastic gradient descent (SGD) algorithms  are among the most effective optimization methods~\cite{bottou2018optimization}.  For strongly-convex functions, i.e., when the smallest eigenvalue of the Hessian matrix is bounded away from \(0\), averaged SGD achieves the
rate of convergence of \(O(1/k)\) after \(k\) iterations~\cite{nemirovski2009robust}.    The stochastic average gradient method (SAG) of \citet{Schmidt2012stochastic} optimizes
the sum of  \(n\)  convex functions with a linear convergence rate   (i.e., a rate that decreases exponentially with the number of iterations)   in the strongly-convex case. Alternative variance-reduced SGD algorithms include stochastic dual coordinate ascent (SDCA) \cite{ZhangShalev2013stochastic} and stochastic
variance reduced gradient (SVRG) \cite{Zhang2013SVRG}. SGD and variance-reduced SGD algorithms typically minimize a real-valued function \(f\) of the form  \begin{equation}
f(\theta)=\frac{1}{n}\sum^{n}_{i=1}f_{i}(\theta),
\end{equation}where \(\theta\) is a \(d\)-dimensional column vector and each \(f_{i}\) is convex and \(L_{i}\)-smooth. In least-squares regressions, for instance, \(f_{i}(\theta)=(x_{i}\theta-y_{i})^{2}/2\), where \(x_{i}\) is the \(i\)-th row of an \(n\times d\) matrix   \(X\), and \(y_{i}\) is the \(i\)-th coordinate of an   \(n\)-dimensional column vector \(Y\). Using similar notation, for ridge regressions,  \(f_{i}(\theta)=(x_{i}\theta-y_{i})^{2}/2+ {\lambda}||\theta||^{2}/2\), where \(\lambda\) is a positive constant. We will concentrate on the case where \(f\) is \(\mu\)-convex. To approximately minimize \(f\), SGD uses the recursion 
\begin{equation}\label{eq:SGDGen}
\theta_{k+1}:=\theta_{k}-\alpha f'_{i}(\theta_{k}),
\end{equation} \(k\geq0\), where \(\alpha\) is a suitable step-size and \(i\) is chosen uniformly at random in \(\{1,\dots,n\}\). Basic variance-reduced SGD algorithms replace \( f'_{i}\) in~\eqref{eq:SGDGen} with a variance-reduced stochastic gradient. For instance, the SVRG algorithm uses several epochs. At the beginning of each epoch, a full gradient is calculated and is used to generate variance-reduced stochastic gradients   throughout the epoch. Basic variance-reduced algorithms typically achieve precision \(\epsilon\) in \(O((n+\kappa_{\max})\log (1/\epsilon))\) stochastic gradient computations, where \(\kappa_{\max}:=\max_{i}(L_{i}/\mu)\). \citet{xiao2014proximal} show that a version of  SVRG with non-uniform sampling achieves precision \(\epsilon\) in \(O((n+\kappa_{\text{avg}})\log (1/\epsilon))\) stochastic gradient computations, where \(\kappa_{\text{avg}}:=(1/n)\sum^{n}_{i=1}L_{i}/\mu\).  Variants of  SVRG  are analysed in \cite{Lacoste-JulienNIPS2015,allen2016improved,lei2017less,BachSVRG2019,kulunchakov2020estimate}. Recently, \citet{kovalev2020} developed L-SVRG, a loopless version of SVRG that does not require the knowledge of the  condition number and achieves precision \(\epsilon\) in \(O((n+\kappa_{\max})\log (1/\epsilon))\) stochastic gradient computations. Accelerated variance-reduced SGD methods are analysed in~\cite{ZhangShalev2014accelerated,nitandaAccelerationNIPS2014,allen2018katyusha,LanZhouAccel2018,kovalev2020}. In particular, \citet{allen2018katyusha} provides an accelerated extension of SVRG that achieves precision \(\epsilon\)   in \(O((n+\sqrt{n\kappa_{\max}})\log (1/\epsilon))\)  stochastic gradient computations. This matches the lower bound given by~\citet{LanZhouAccel2018}.
\citet{gower2020variance} provide a recent review on variance-reduced optimization algorithms for machine learning.

Several optimization problems that arise in scientific computing and data analysis involve the minimization of quadratic functions of the form 
\begin{equation}
\label{eq:QuadraticMinimization}f(\theta)=\frac{1}{2}\theta^{T}H\theta-c^{T}\theta,
\end{equation}
where \(H\) is a   positive definite \(d\times d\) matrix, \(c\) is a   \(d\)-dimensional column vector, and \(\theta\) ranges over all \(d\)-dimensional column vectors. This paper assumes  that   \(f\) is \(\mu\)-convex with \(\mu>0\), and that \(H=E(Q)\), where \(Q\) is  a random symmetric   \(d\times d\) matrix with \(E(Q^{2})\leq LH\), where \(L>0\) is a known constant. The  latter condition holds if the eigenvalues of \(Q\) are between \(0\) and \(L\). As shown in Section~\ref{se:examples}, this framework is suitable to several applications  including least-squares regressions, ridge regressions,   linear discriminant analysis and regularized linear discriminant analysis.  These methods are   widely used for inference and prediction\comment{ (e.g. \cite{bartram2018agnostic}}, and many of the modern  machine learning techniques such as the logistic regression, the lasso method and neural networks can be considered as extensions of these techniques. As \(f'(\theta)=H\theta-c\), a stochastic gradient of \(f\) is \(Q\theta-c\). We introduce Q-SVRG, a variant of the SVRG algorithm adapted to quadratic functions, to minimize
 \(f\) at an arbitrary precision.  The algorithm   Q-SVRG does not require the explicit computation of \(H\). It is similar to  the variant of SVRG developed by \citet{xiao2014proximal}, and consists of  \(l\) epochs each comprising  \(m\) inner iterations. A full gradient is calculated at the beginning of every epoch, and each inner iteration calculates a stochastic gradient. In the aforementioned applications,  the calculation of the full gradient takes \(O(nd)\) time, and the  running time of each inner iteration of Q-SVRG is  \(O(d)\).  Non-uniform sampling is handled by    Q-SVRG in a straightforward manner by choosing a non-uniform distribution for \(Q\). The theoretical design and analysis of    SVRG in \cite{xiao2014proximal} and of Q-SVRG is based on the average of iterates, and so do our numerical experiments on Q-SVRG.  \citet{xiao2014proximal} use however the last iterate in their numerical experiments on SVRG.   Q-SVRG   enjoys the following properties:
\begin{enumerate}
\item For any fixed  \(l\ge1\),   Q-SVRG achieves a
rate of convergence  of  \(O((\kappa /m)^{l})\) with  \(l\) epochs comprising  \(m\) inner iterations each,  where \(\kappa:=L/\mu\) is the condition number, and can be simulated  without the knowledge of \(\mu\). In contrast, for any fixed number of epochs,
the analysis in \cite{xiao2014proximal} implies a constant optimality gap when the number of inner iterations goes to infinity. For  least-squares and ridge regressions with constant \(\kappa\), Q-SVRG achieves precision \(\epsilon=O(n^{-l})\) in \(O(nd)\) total time. We are not aware of any previous result  showing such time-accuracy tradeoff. \item When  \(\mu\) is known and the calculation of a full gradient of \(f\) requires \(n\) stochastic gradients,  Q-SVRG has linear convergence and achieves precision \(\epsilon\) in \begin{displaymath}
O(\frac{n+\kappa}{\max(1,\log ( n/\kappa))}\log \frac{1}{\epsilon})
\end{displaymath}
stochastic gradient computations. This improves the state-of-the-art performance of SVRG~\cite{xiao2014proximal} up to a logarithmic factor. When \(\kappa\) is constant, the gradient-complexity of Q-SVRG outperforms that of  SVRG in \cite{xiao2014proximal}, accelerated SVRG in \cite{allen2018katyusha}, and is better than the aforementioned  lower bound of~\citet{LanZhouAccel2018},   by a factor of order \(\log (n)\).  This lower bound was established for a quadratic objective function that falls into the \eqref{eq:QuadraticMinimization} framework and  a class of randomized  gradient methods which
sequentially calculate the gradient of a random component function. This class does not include, stricto-sensu,    Q-SVRG, which acquires  the full gradient at the beginning of every epoch. Thus the \citet{LanZhouAccel2018} complexity lower bound does not apply to  Q-SVRG.\item  The theoretical analysis of Q-SVRG allows a step-size of length up to \(1/L\). This is larger by a multiplicative constant than step-sizes given by previous theoretical analysis  of  variance-reduced algorithms such as SAG~\cite{BachSchmidt2017minimizing}, SVRG~\cite{Zhang2013SVRG,xiao2014proximal} or L-SVRG~\cite{kovalev2020}.
\end{enumerate}
\comment{Improving running times by logarithmic factors has been the source of  major studies in the variance-reduced SGD literature (see, e.g., \cite{allen2018katyusha}). }
\subsection{Other related work}
An exact solution to the least-squares regression problem can be found in \(O(nd^{2})\) time \cite{golub2013matrix}. \citet{rokhlin2008fast} describe a randomized algorithm based on a preconditioning matrix that minimizes least-squares regressions with  relative precision \(\epsilon\) in \(O(d^{3}+nd\log (d/\epsilon))\) time, for \(\epsilon>0\).     For non-strongly-convex linear regressions, \citet{bachMoulines2013non} show   a convergence rate
   of \(O(1/k)\) after \(k\) iterations for an averaged SGD
algorithm with constant step-size in an on-line setting. 
\citet{Richtarik2015randomized} describe a randomized iterative
method with a linear rate for solving linear systems, that is also applicable to least squares regressions. However, when applied to least squares,  a naive implementation of their iterative step requires  $O(nd)$ time. This  can be reduced to \(O(d)\)  time if \(X^{T}X\) is precomputed, which takes \(O(nd^{2})\) time.
 \citet{pilanci2016iterative} provide algorithms  for   constrained least-squares through a random projection on a lower dimensional space. They show how to minimize least-squares regressions
with relative precision \(\epsilon\) in \(O((nd\log (d)+d^{3})\log (1/\epsilon))\) time.   \citet{BachLeastSquaresJMLR2017} study an averaged accelerated regularized
SGD algorithm for least-squares regressions. Mini-batching and tail-averaging SGD algorithms for least-squares regressions are analyzed by \citet{jain2018parallelizing}. Our updating rule is similar to the recursion used by 
\citet{kahale2019Gaussian}  to approximately simulate high-dimensional Gaussian vectors with a given covariance matrix.
 \citet{loizou2020momentum}   provide stochastic algorithms with linear rates to minimize  the expectation of a random quadratic function. In their framework, however,      stochastic gradients vanish at the optimum, which is not the case in our setting.       
The remainder of the paper is organized as follows. Section~\ref{se:AlgorithmDescription} describes Q-SVRG and its properties. Section~\ref{se:examples} describes applications of Q-SVRG. Section~\ref{se:numer} gives numerical experiments. Section~\ref{se:conclusion} contains concluding remarks. Omitted proofs are in \commentt{the Supplementary Material}{the supplementary material}. The running time refers to the number of arithmetic operations.

\comment{ \citet{frostig2015regularizing} give an algorithm that minimizes \(g(\theta)\) in time \(O(dn\sqrt{\kappa}\log (\epsilon_{0}/\epsilon))\), where \(\kappa=\lceil R^{2}/\lambda_{\min}(X^{T}X)\rceil\) is the condition number, \(R\) is the largest Euclidean norm of a row of \(X\), and \(\epsilon_0/\epsilon\) is the ratio between the initial and desired accuracy. } 
\section{The  algorithm description and properties}\label{se:AlgorithmDescription}
Let \(I\) denote the \(d\times d\) identity matrix.
This section makes the following assumptions.\begin{description}
\item[Assumption 1 (A1).] \(H\geq \mu I\), where \(\mu\) is a positive constant.
\item[Assumption 2 (A2).] 
There is a random sequence \((Q_{k}:k\ge0)\) of independent  symmetric  \(d\times d\) matrices such that  \(E(Q_{k})=H\) and \(E({Q_{k}}^{2})\le LH\).
\end{description} 
Define the condition number \(\kappa:=L/\mu\). Let  \(t_{H}\) (resp. \(t_{Q}\)) be the time needed to calculate  \(H\theta\)    (resp. \(Q_{k}\theta\)) for a given  \(d\)-dimensional vector \(\theta\). Assumption A1 implies that \(H\) is invertible and that \(\mu\) is smaller than the smallest eigenvalue of \(H\).  A2 implies that \(H\leq LI\). Conversely,  if A1 holds and  \(H\leq LI\), then A2 trivially holds by choosing \(Q_{k}=H\) for \(k\ge0\). Our applications, though, use matrices \(Q_{k}\) such that \(t_{H}\) is much larger than \(t_{Q}\), which is typically of order \(d\).

Given an initial  \(d\)-dimensional column vector \(\theta_{0}\) and a real number \(\alpha\in(0,1/L]\), define the sequence of  \(d\)-dimensional column vectors
\((\theta_k:k\ge0)\) via the recursion
\begin{equation}\label{eq:BasicDefthetaGen}
\theta_{k+1}:=\theta_{k}-\alpha(Q_{k}(\theta_{k}-\theta_{0})-c+H\theta_{0}),
\end{equation}
for \(k\ge0\). As \(f'(\theta)=H\theta-c\) and, by A2, \begin{displaymath}
E(Q_{k}(\theta_{k}-\theta_{0})-c+H\theta_{0}|\theta_{k})=H\theta_{k}-c,
\end{displaymath}
\eqref{eq:BasicDefthetaGen} can be viewed as a variant of SGD. Since \(Q_{k}(\theta_{k}-\theta_{0})\) is equal to the difference between the stochastic gradients \(Q_{k}\theta_{k}-c\) and \(Q_{k}\theta_{0}-c\),  \eqref{eq:BasicDefthetaGen} can be considered as a variance-reduced SGD, and is essentially the same recursion used in the inner iteration of SVRG~\cite{Zhang2013SVRG}.
 For \(k\geq1\), let \begin{equation*}
\bar\theta_{k}:=\frac{\theta_{0}+\cdots+\theta_{k-1}}{k}.
\end{equation*}
Thus, the calculation of  \(\bar\theta_{k}\) takes \(O(t_{H}+kd +kt_{Q})\) time and involves one full gradient and \(k\) stochastic gradient computations (following the literature convention on variance-reduced SGD methods applied to quadratic problems, we consider that calculating \(Q_{k}(\theta_{k}-\theta_{0})\) involves one, rather than two, stochastic gradient computations).

Given \(m\ge1\), let \(T_{m}\) be the random operator that maps any \(d\)-dimensional column vector  \(\theta_{0}\) to \(\bar\theta_{m}\). For \(l\ge1\), denote by  \(T^{l}_{m}\)    the random operator on the set of  \(d\)-dimensional column vectors obtained by composing \(l\) times the operator  \(T_{m}\). Thus, calculating    
 \(T^{l}_{m}(\theta)\)  from \(\theta\) takes  \(O(l(t_{H}+m d +m t_{Q}))\) time and involves \(l\) full gradient and \(l m\) stochastic gradient computations. Algorithm~\ref{alg:RHAlk} gives a pseudo-code that outputs \(T^{l}_{m}(0)\).
%\begin{minipage}[t]{0.46\textwidth}
\begin{algorithm}[H]
\caption{Procedure Q-SVRG}
\label{alg:RHAlk}
\begin{algorithmic}
\Procedure{Q-SVRG}{$\alpha,m,l$}
\State\(\theta_{0}\gets0\) 
\For{$h\gets 1,l$}
\State \(\tilde c\gets c-H\theta_{0}\)
\For{$k\gets0,m-1$}    
\State $\theta_{k+1}=\theta_{k}-\alpha(Q_{k}(\theta_{k}-\theta_{0})-\tilde c)$
\EndFor
\State  
$\theta_{0}\gets(\theta_{0}+\cdots+\theta_{m-1})/{m}$
\EndFor
\State \Return{$\theta_{0}$} 
\EndProcedure
\end{algorithmic}
\end{algorithm}
\comment{
\end{minipage}
\hfill
\begin{minipage}[t]{0.46\textwidth}
\begin{algorithm}[H]
\caption{Procedure Q-SVRG for least squares regression}
\label{alg:RHAlkLeastSquares}
\begin{algorithmic}
\Procedure{Q-SVRG}{$\alpha,m,l$}
\For{$i\gets1,n$}    
\State\(p_{i}\gets {||X^{T}e_{i}||^{2}}/{\tr(X^{T}X)}\) 
\EndFor
\State\(\theta_{0}\gets0\) 
\For{$h\gets 1,l$}
\State \(\tilde c\gets\tr(X^{T}X)^{-1}X^{T}(y-X\theta_{0})\)
\For{$k\gets0,m-1$}    
\State Sample  \(i(k)\) from \(\{1,\dots,n\}\) such that $\Pr(i(k)=j)=p_{j}$ for \(1\leq j\leq n\)  
\State  \(u_{k}\gets||X^{T}e_{i(k)}||^{-1}(X^{T}e_{i(k)})\)    
\State $\theta_{k+1}=\theta_{k}-\alpha (u_{k}(u_{k}^{T}(\theta_{k}-\theta_{0}))-\tilde c)$
\EndFor
\State  
$\theta_{0}\gets(\theta_{0}+\cdots+\theta_{m-1})/{m}$
\EndFor
\State \Return{$\theta_{0}$} 
\EndProcedure
\end{algorithmic}
\end{algorithm}
\end{minipage}
}
\subsection{The worst-case analysis}\label{sub:analysis}
As \(H\) is invertible, there is a unique   \(d\)-dimensional column vector \(\theta^{*}\) such that
\begin{equation}\label{eq:Htheta*}
H\theta^{*}=c.
\end{equation}
By a standard calculation,  for any    \(d\)-dimensional column vector \(\theta\), 
\begin{equation}\label{eq:minimumf}
f(\theta)-f(\theta^{*})=\frac{1}{2}(\theta-\theta^{*})H(\theta-\theta^{*}),
\end{equation} 
and so \(f\)  attains its minimum  at \(\theta^{*}\). Theorem~\ref{th:mainGen}  analyses the convergence properties of Q-SVRG.
\begin{theorem}\label{th:mainGen}Assume that A1 and A2 hold. Let \(\theta^{*}\) be the unique \(d\)-dimensional vector satisfying \eqref{eq:Htheta*}. Then, for any  \(\theta\in\mathbb{R}^{d}\),  \(\alpha\in(0,1/L]\), \(l\ge1\), and  \(m\ge1\),
\begin{equation}\label{eq:thMainPositiveDefinite}
E(f(T^{l}_{m}(\theta)))-f(\theta^{*})\leq(\frac{9}{\alpha\mu m} )^{l}(f(\theta)-f(\theta^{*})).
\end{equation}
\end{theorem}
The proof of Theorem~\ref{th:mainGen} follows by  induction on \(l\) and bounding separately  bias and variance terms. 
As \eqref{eq:minimumf} implies that \begin{equation}
||\theta-\theta^{*}||^{2}\leq\frac{2}{\mu}(f(\theta)-f(\theta^{*})),
\end{equation} it follows from Theorem~\ref{th:mainGen} that 
\begin{equation*}
E(||T^{l}_{m}(\theta_{0})-\theta^{*}||^{2})\leq\frac{2}{\mu}(\frac{9}{\alpha\mu m} )^{l}(f(\theta_{0})-f(\theta^{*})).
\end{equation*}  When A1 and A2 hold and   \(\alpha=1/L\), Theorem~\ref{th:mainGen} implies that   \(E(f(T^{l}_{m}(\theta_{0})))-f(\theta^{*})=O((\kappa/m)^{l})\) as \(m\) goes to infinity, for any fixed \(l\ge1\).  Observe that  \(T^{l}_{m}(\theta)\)  can be simulated without the explicit knowledge of \(\mu\). 

Assume now that   \(\mu\) is known and that one full gradient is a weighted sum of   \(n\) stochastic gradients, with known weights. The latter condition implies that each epoch with \(m\) inner iterations involves \(n+m\) stochastic gradient calculations. This condition holds in several applications with \(n\) data points (see Section~\ref{se:examples}). Set \(\alpha=1/L\).  By Theorem~\ref{th:mainGen}, Q-SVRG  minimizes \(f\) with expected error \(\epsilon\) with \(l\) epochs, each containing  \(m\) inner iterations,  where  \(m=9\max(e\kappa,n)\) and \begin{displaymath}
l=\frac{1}{\max(1,\log ( n/\kappa))}\log\frac{f(\theta_{0})-f(\theta^{*})}{\epsilon}.
\end{displaymath}  
Thus, the total number of stochastic gradient computations required by  Q-SVRG   to minimize \(f\) with expected error \(\epsilon\) is\begin{eqnarray*}
N_{\epsilon}&\le&l(n+m)\\&=&\frac{n+\kappa}{\max(1,\log ( n/\kappa))}\log\frac{f(\theta_{0})-f(\theta^{*})}{\epsilon}.
\end{eqnarray*}  This improves  the  \(O((n+\kappa)\log(1/\epsilon))\)  gradient-complexity of SVRG~\cite{xiao2014proximal} up to a logarithmic factor. In particular, when \(\kappa\) is constant, \begin{displaymath}
N_{\epsilon}=O(\frac{n}{\log (n)}\log\frac{1}{\epsilon}),
\end{displaymath}which  is better than the gradient-complexity of  SVRG in \cite{xiao2014proximal}, accelerated SVRG in \cite{allen2018katyusha}, and  the aforementioned  lower bound of~\citet{LanZhouAccel2018},   by a factor of order \(\log (n)\).

\comment{Similarly, in applications involving \(n\) data points in dimension \(d\) with \(t_{H}=O(nd)\) and \(t_{Q}=O(d)\), (see Section~\ref{se:examples}), setting \(\alpha=1/L\) and \(m=\max(9e\kappa,n)\) implies that $$\tau_{\epsilon}=O(\frac{(n+\kappa)d}{\max(1,\log ( n/(9\kappa)))}\log\frac{f(\theta_{0})-f(\theta^{*})}{\epsilon}).$$}
\section{Examples}\label{se:examples}
This section gives examples where A1 and A2 hold. 
\subsection{Least-squares regression}\label{sub:leastSquares}
Given an  \(n\times d\) matrix   \(X\) with rank \(d\)    and an   \(n\)-dimensional column vector \(Y\), the least-squares regression consists of  minimizing the function \(g(\theta):=(2n)^{-1}||X\theta-Y||^{2}\), where \(\theta\) ranges over all \(d\)-dimensional column vectors. This problem can be reduced to~\eqref{eq:QuadraticMinimization} by setting  \(H:=\tr(X^{T}X)^{-1}X^{T}X\) and  \(c:=\tr(X^{T}X)^{-1}X^{T}Y\), which implies that \(g(\theta)=\bar L f(\theta)+g(0)\),  where \(\bar L:=\tr(X^{T}X)/n\) is the average squared norm of a line of \(X\).  As \(H\) and \(X\) have the same rank, \(H\) is invertible and A1 holds with \(\mu\) being the smallest eigenvalue of \(H\). For \(1\leq i\leq n\),  let  \(e_{i}\) be the  \(n\)-dimensional column vector whose \(i\)-th component is \(1\) and remaining components are \(0\), and let
\begin{equation}\label{eq:defPi}
p_{i}=\frac{||X^{T}e_{i}||^{2}}{\tr(X^{T}X)}.
\end{equation}
Note that the numerator in~\eqref{eq:defPi} is the sum of the squared entries of the \(i\)-th line of \(X\), while the denominator is the sum of squared entries of \(X\). Thus the \(p_{i}\)'s sum up to \(1\). Let \((i(k):k\ge0)\) be a sequence of independent integral random variables on \(\{1,\dots,n\}\) such that, for \(1\leq j\leq n\),  
\begin{equation*}\label{eq:ikDef}
\Pr(i(k)=j)=p_{j}.
\end{equation*} 
For \(k\geq0\), let  \(u_{k}:=||X^{T}e_{i(k)}||^{-1}(X^{T}e_{i(k)})\)  and \(Q_{k}:=u_{k}u_{k}^{T}\).
As \(u_{k}\) is a unit vector, the largest eigenvalue of \(Q_{k}\) is equal to \(1\). Furthermore, by the definition of \(u_{k}\),
\begin{eqnarray}\label{eq:EukukT}
E(Q_{k})\nonumber
&=&\sum_{j=1}^n p_{j} ||X^{T}e_{j}||^{-2}(X^{T}e_{j}e_{j}^{T}X)\\
\nonumber
&=&\frac{1}{\tr(X^{T}X)}X^{T}(\sum_{j=1}^n e_{j}e_{j}^{T})X\\
&=&H. 
\end{eqnarray}
The second equation follows from~\eqref{eq:defPi}, and the last one by observing that \(\sum_{j=1}^n e_{j}e_{j}^{T}\) is  the \(n\times n\) identity matrix. Thus A2 holds with \(L=1\).  The recursion~\eqref{eq:BasicDefthetaGen}
becomes
\begin{equation}\label{eq:BasicDeftheta}
\theta_{k+1}=\theta_{k}-\alpha((u_{k}^{T}(\theta_{k}-\theta_{0}))u_{k}-c+H\theta_{0}),
\end{equation} 
 for \(k\ge0\).

As \(X^{T}e_{i}\) is the \(i\)-th column of \(X^{T}\), for \(1\leq i \leq n\),  the total time to calculate \(c\), \(H\theta_{0}\) and the \(p_{i}\)'s is \(O(nd)\).   After an initial preprocessing cost of \(O(n)\), the random variable  \(i(k)\) can be simulated in constant time using the alias method  \cite[Section III.4]{DevroyeSpringer}. Thus   the cost of each iteration in~\eqref{eq:BasicDefthetaGen}  is \(O(d)\). Algorithm~\ref{alg:RHAlkLeastSquares} gives a pseudo-code for the Q-SVRG algorithm applied to least-squares regressions.
As  \(Q_{k}\theta=u_{k}(u_{k}^{T}\theta)\), we have \(t_{Q}=O(d)\). Moreover,  \(t_{H}=O(nd)\). It follows from \eqref{eq:EukukT} that one full gradient is a weighted sum of  \(n\) stochastic gradients.

 The recursion~\eqref{eq:BasicDeftheta} uses a non-uniform sampling scheme with sampling probabilities determined by the squared norm of each row vector. A similar sampling scheme has been applied by~\citet{frieze2004fast} in the context of low-rank approximations of a matrix, by  \citet{strohmer2009randomized} to approximately solve linear systems via an iterative algorithm, and by  \citet{BachDefossezAIStat2015averaged} to design an averaged SGD for least-squares regressions.   However, while the updating rule of  the conventional SGD, of   \citet{strohmer2009randomized} and of  \citet{BachDefossezAIStat2015averaged} uses a single random coordinate of \(Y\), \eqref{eq:BasicDeftheta} uses the vector \(c\) that depends on the entire vector \(Y\).  \citet{strohmer2009randomized} 
establish a linear convergence rate for their method in the strongly-convex case.   \citet{BachDefossezAIStat2015averaged}  give a detailed asymptotic
analysis (as the number of iterations goes to infinity) of their algorithm.
\begin{algorithm}[!t]
\caption{Procedure Q-SVRG for least squares regression}
\label{alg:RHAlkLeastSquares}
\begin{algorithmic}
\Procedure{Q-SVRG}{$\alpha,m,l$}
\For{$i\gets1,n$}    
\State\(p_{i}\gets {||X^{T}e_{i}||^{2}}/{\tr(X^{T}X)}\) 
\EndFor
\State\(\theta_{0}\gets0\) 
\For{$h\gets 1,l$}
\State \(\tilde c\gets\tr(X^{T}X)^{-1}X^{T}(y-X\theta_{0})\)
\For{$k\gets0,m-1$}    
\State Sample  \(i(k)\) from \(\{1,\dots,n\}\) such that $\Pr(i(k)=j)=p_{j}$ for \(1\leq j\leq n\)  
\State  \(u_{k}\gets||X^{T}e_{i(k)}||^{-1}(X^{T}e_{i(k)})\)    
\State $\theta_{k+1}=\theta_{k}-\alpha (u_{k}(u_{k}^{T}(\theta_{k}-\theta_{0}))-\tilde c)$
\EndFor
\State  
$\theta_{0}\gets(\theta_{0}+\cdots+\theta_{m-1})/{m}$
\EndFor
\State \Return{$\theta_{0}$} 
\EndProcedure
\end{algorithmic}
\end{algorithm}

\subsection{Ridge regression}\label{subsub:ridge}
Given a  non-zero \(n\times d\) matrix   \(X\), an   \(n\)-dimensional column vector \(Y\), and \(\lambda>0\), the ridge regression consists of minimizing the function
\begin{displaymath}
g(\theta):=\frac{1}{2n}||X\theta-Y||^{2}+\frac{\lambda}{2}||\theta||^{2},
\end{displaymath}  
where \(\theta\) ranges over all \(d\)-dimensional column vectors. This problem can be reduced to~\eqref{eq:QuadraticMinimization} by setting  \begin{displaymath}
H:=(\lambda +\bar L)^{-1}(\lambda I+n^{-1}X^{T}X)
\end{displaymath} and  \(c:=(\lambda n +\bar Ln)^{-1}X^{T}Y\), where \(\bar L:=\tr(X^{T}X)/n\), which implies that \(g(\theta)=(\lambda +\bar L) f(\theta)+g(0)\). As \(X^{T}X\) is symmetric positive semidefinite, A1 holds with \(\mu=\lambda/(\lambda +\bar L)\). Let \begin{displaymath}
Q_{k}:=(\lambda+\bar L)^{-1}(\lambda I+\bar Lu_{k}u_{k}^{T}),
\end{displaymath}where \(u_{k}\) is defined as in Section~\ref{sub:leastSquares}. Then \(Q_{k}\leq I\) and  \(E(Q_{k})=H\). Thus A2 holds with \(L=1\). An analysis similar to the one in Section~\ref{sub:leastSquares} shows that, after a total preprocessing cost of \(O(nd)\),   the cost of each iteration in~\eqref{eq:BasicDefthetaGen} is \(O(d)\).  Furthermore,  \(t_{Q}=O(d)\) and \(t_{H}=O(nd)\), and one full gradient is a weighted sum of \(n\) stochastic gradients.
Thus the Q-SVRG algorithm with \(\alpha=1\) and \(m=9\max(e(\lambda +\bar L)/\lambda,n)\) has a linear convergence rate.
\subsection{Linear discriminant analysis}
Consider \(d\)-dimensional column vectors \(x_{1},\dots,x_{n}\), where \(x_{i}\) belongs to class \(g(i)\), with \(g(i)\in\{1,\dots,K\}\). For \(1\leq k\leq K\), let \(n_{k}\) be the number of observations in class \(k\), and let \begin{displaymath}
\hat\mu_{k}:=\frac{1}{n_{k}}\sum_{i:g(i)=k}x_{i}
\end{displaymath} be their average. Assume that the \(d\times d\) matrix \begin{displaymath}
\hat \Sigma :=\frac{1}{n-K}\sum^{n}_{i=1}(x_{i}-\hat\mu_{g(i)})(x_{i}-\hat\mu_{g(i)})^{T}
\end{displaymath}
 is invertible. The linear discriminant analysis  method~\cite[Section 4.3]{hastieTibshirani2009elements} classifies a \(d\)-dimensional column vector \(x\) by calculating the linear discriminant functions\begin{displaymath}
\delta_{k}(x)=(x-\frac{1}{2}\hat\mu_{k})^{T}\hat\Sigma^{-1}\hat\mu_{k}+\log(n_{k}/n),
\end{displaymath} 
 \(1\leq k\leq K\). Then \(x\) is assigned to \(\arg\max_{k}\delta_{k}(x)\). Note that \begin{displaymath}
\tr(\hat \Sigma)=\frac{1}{n-K}\sum^{n}_{i=1}(x_{i}-\hat\mu_{g(i)})^{T}(x_{i}-\hat\mu_{g(i)})
\end{displaymath}
can be calculated in \(O(nd)\) time.
Given  \(k\in\{1,\dots,K\}\), let \(H:=\tr(\hat \Sigma)^{-1}\hat \Sigma\) and \(c:=\tr(\hat \Sigma)^{-1}\hat\mu_{k}\). Then\begin{displaymath}
\hat\Sigma^{-1}\hat\mu_{k}=\arg\min f,
\end{displaymath} 
where \(f\) is given by \eqref{eq:QuadraticMinimization}.
As \(\hat \Sigma =X^{T}X\), where \(X\) is the \(n\times d\) matrix whose \(i\)-th line is \((n-K)^{-1/2}(x_{i}-\hat\mu_{g(i)})^{T}\), the function \(f\) can be minimized via~\eqref{eq:BasicDefthetaGen}
using the approach outlined in Section~\ref{sub:leastSquares}. Here again,   \(t_{Q}=O(d)\) and \(t_{H}=O(nd)\), and one full gradient is a weighted sum of  \(n\) stochastic gradients. Regularized linear discriminant analysis can be treated in a similar way.
\section{Numerical experiments}\label{se:numer}
Our numerical experiments were conducted for ridge regressions on the sonar\footnote{http://archive.ics.uci.edu\label{ft:uci}}, madelon\textsuperscript{\ref{ft:uci}} and sido0\footnote{http://www.causality.inf.ethz.ch} binary datasets, whose characteristics are summarized in Table~\ref{tab:datasets}. The variables were centered,  a constant variable was added  to each dataset, and all variables were normalized.  
\begin{table} \caption{datasets used in the simulations}
\center
%\begin{tiny}
\begin{tabular}{lrr}\hline
  dataset  & Variables     &Data Points \\ \hline
sonar &$60$ & $ 208$  \\
madelon &$500$ & $2000$  \\
sido0 &$4932$ & $12678$  \\
\hline
 \end{tabular}
 %\end{tiny}
 \label{tab:datasets}
\end{table}
  The codes were written in the C++ programming language,
 the compiler used was Microsoft Visual C++ 2013, and the experiments were performed on a laptop PC with an Intel  processor and 8 GB of RAM running Windows 10 Professional. For each dataset, we have implemented  the following methods using the null vector as starting point and  the notation in Section~\ref{subsub:ridge}: \begin{itemize}
\item 
the averaged SGD algorithm with  uniform sampling and step-size  \(1/(4(\lambda+\max_{1\leq i\leq n}||X^{T}e_{i}||^{2}))\), adapted from~\cite{bachMoulines2013non}.
\item the averaged SGD algorithm with non-uniform probabilities adapted from~\cite{BachDefossezAIStat2015averaged}, with step-size  \(1/(\lambda+\bar  L)\).

\item the SAG algorithm with non-uniform probabilities adapted from~\cite{BachSchmidt2017minimizing}, with step-size  \(1/(\lambda+\bar  L)\),  where the lines are sampled according  to the \(p_{i}\)'s, and  the output is the vector among the final iterate and the average of iterates that minimizes \(g\).
\item the SVRG algorithm with non-uniform probabilities. Following the experimental recommendations of \citet{xiao2014proximal}, each epoch comprises  \(2n\) inner iterations, outputs the final iterate,  and the step-size is \(0.1/(\lambda+\bar  L)\).
\item the L-SVRG  method with  uniform sampling as described by  \citet{kovalev2020},  with average epoch-length \(n\) and step-size  \(1/(6(\lambda+\max_{1\leq i\leq n}||X^{T}e_{i}||^{2}))\).
\item the Q-SVRG method that approximately minimizes \(f\) by calculating \(T^{l}_{m}(0)\) via  Algorithm~\ref{alg:RHAlk}, with  \(l=\max(4,N\min(1/n,\lambda/\bar L))\), where \(N\) is the  target total number of inner iterations,  \(m=\lfloor N/l\rfloor\), and  \(\alpha=1\).  Thus, ignoring integrality constraints, \(m=\min(N/4,\max(n,\bar L/\lambda))\). The second argument of the min function is within a constant factor from the value of \(m\) suggested in Section~\ref{subsub:ridge}. The lower bound \(4\) on \(l\) was chosen after running a few computer simulations.  
 \end{itemize}
The results are reported in Figure~\ref{fig:sido0lambda0.01}. The running time is measured by  the number of effective passes, defined as the total number of stochastic gradients divided by \(n\). Each iteration of  the averaged  SGD and SAG methods accounts for one stochastic gradient,  while each epoch of the SVRG, L-SVRG and  Q-SVRG algorithms containing \(m\)  inner iterations accounts for \(n+m\) stochastic gradients. In all our computer experiments, when the number of effective passes is sufficiently large, Q-SVRG outperforms the averaged SGD, SVRG, L-SVRG and SAG methods, expect for the sonar dataset, where Q-SVRG is sometimes outperformed by SAG. 
\begin{figure}
\centering
\begin{subfigure}[b]{0.3\textwidth}
\centering
\begin{tikzpicture}[scale=.45]
\begin{axis}[
    title={},
    ymode=log,
    xlabel={Number of effective passes},
    ylabel={$g(\theta)-g(\theta^*)$},
    legend pos=south west,
    ymajorgrids=true,
    grid style=dashed,
]
 \addplot[
    color=magenta,
    mark=o,
    ]
    coordinates {
( 9,0.0145612)  
( 13,0.00948955)  
( 19,0.00632973)  
( 27,0.00408148)  
( 39,0.00260759)  
( 55,0.00138561)  
    };
 \addplot[
    color=red,
    mark=diamond,
    ]
    coordinates {
( 9,0.00739444)  
( 13,0.00357817)  
( 19,0.0024758)  
( 27,0.00164066)  
( 39,0.00162961)  
( 55,0.00104551)  
    };
 \addplot[
    color=darkgray,
    mark=star,
    ]
    coordinates {
( 9,0.000566802)  
( 13,0.000115369)  
( 19,2.16237e-006)  
( 27,2.61483e-008)  
( 39,6.14091e-010)  
( 55,3.27294e-013)  
    };
 \addplot[
    color=blue,
    mark=square,
    ]
    coordinates {
( 9,0.003348251213)  
( 12,0.001795211678)  
( 18,0.0005750424435)  
( 27,0.0001196822107)  
( 39,1.667955328e-005)  
( 54,1.573397128e-006)  
}; 
\addplot[
    color=orange,
    mark=otimes,
    ]
    coordinates {
( 8,0.0596764)  
( 10,0.0448737)  
( 12,0.0327391)  
( 16,0.0231632)  
( 24,0.0157318)  
( 32,0.0102829)  
( 57,0.00642519)  
};    
\addplot[
    color=green,
    mark=triangle,
    ]
    coordinates {
( 8,0.0001074290863)  
( 12,8.831131807e-006)  
( 18,3.352761752e-007)  
( 26,4.291846489e-009)  
( 38,1.441557984e-011)  
( 54,5.162537065e-015)  
};
%\legend{SGD, N.U. SGD, N.U. SAG, N.U. SVRG, LSVRG, Q-SVRG}
\end{axis}
\end{tikzpicture}
%\hspace{1cm}
%\caption{Convergence on sonar dataset with \(\lambda=\bar L/n\)}
\label{fig:sonarlambda1}
\end{subfigure}
\begin{subfigure}[b]{0.3\textwidth}
\centering
\begin{tikzpicture}[scale=.45]
\begin{axis}[
    title={},
    ymode=log,
    xlabel={Number of effective passes},
    ylabel={$g(\theta)-g(\theta^*)$},
    legend pos=south west,
    ymajorgrids=true,
    grid style=dashed,
]
 \addplot[
    color=magenta,
    mark=o,
    ]
    coordinates {
( 8,0.00912444)  
( 11,0.0062398)  
( 16,0.00412626)  
( 23,0.00292784)  
( 32,0.00217446)  
( 46,0.00139727)  
    };
 \addplot[
    color=red,
    mark=diamond,
    ]
    coordinates {
( 8,0.0119633)  
( 11,0.00807373)  
( 16,0.0061898)  
( 23,0.00409503)  
( 32,0.00292157)  
( 46,0.00195437)  
    };
 \addplot[
    color=darkgray,
    mark=star,
    ]
    coordinates {
( 8,0.00242052)  
( 11,0.000448204)  
( 16,3.44401e-005)  
( 23,1.29844e-006)  
( 32,5.52984e-009)  
( 46,6.49131e-012)  
    };
 \addplot[
    color=blue,
    mark=square,
    ]
    coordinates {
( 9,0.001026639239)  
( 15,0.0001161494782)  
( 21,1.872795412e-005)  
( 30,2.282978596e-006)  
( 45,2.040377371e-007)  
}; 
\addplot[
    color=orange,
    mark=otimes,
    ]
    coordinates {
( 8,0.0137559)  
( 11,0.00823149)  
( 16,0.00457647)  
( 24,0.00240943)  
( 31,0.00122763)  
( 44,0.000616237) 
};    
\addplot[
    color=green,
    mark=triangle,
    ]
    coordinates {
( 8,4.761162406e-005)  
( 10,5.35927448e-006)  
( 16,5.607507664e-008)  
( 22,1.605952593e-009)  
( 32,1.655309223e-011)  
( 46,2.797762022e-014)  
};
%\legend{SGD, N.U. SGD,N.U. SAG,N.U. SVRG, LSVRG, Q-SVRG}
\end{axis}
\end{tikzpicture}
%\hspace{1cm}
%\caption{Convergence on Madelon dataset with \(\lambda=\bar L/n\)}
\label{fig:madelonLambda1}
\end{subfigure}
\begin{subfigure}[b]{0.3\textwidth}
\centering
\begin{tikzpicture}[scale=.45]
\begin{axis}[
    title={},
    ymode=log,
    xlabel={Number of effective passes},
    ylabel={$g(\theta)-g(\theta^*)$},
    legend pos=south west,
    ymajorgrids=true,
    grid style=dashed,
]
 \addplot[
    color=magenta,
    mark=o,
    ]
    coordinates {
( 7,0.0996867)  
( 10,0.0663855)  
( 14,0.0402041)  
( 20,0.0223882)  
( 29,0.0117356)  
( 41,0.00596797)  
    };
 \addplot[
    color=red,
    mark=diamond,
    ]
    coordinates {
( 7,0.00321603)  
( 10,0.00209563)  
( 14,0.00140854)  
( 20,0.000940404)  
( 29,0.000599859)  
( 41,0.00040214)  
    };
 \addplot[
    color=darkgray,
    mark=star,
    ]
    coordinates {
( 7,0.00257491)  
( 10,0.000657259)  
( 14,8.00392e-005)  
( 20,4.20596e-006)  
( 29,6.51544e-008)  
( 41,1.92827e-010)  
    };
 \addplot[
    color=blue,
    mark=square,
    ]
    coordinates {
( 6,0.01999265019)  
( 9,0.005127187258)  
( 12,0.001401948167)  
( 18,0.0001518327776)  
( 27,1.634989216e-005)  
( 39,1.818571481e-006)  
}; 
\addplot[
    color=orange,
    mark=otimes,
    ]
    coordinates {
( 9,0.176599)  
( 12,0.140204)  
( 18,0.103189)  
( 27,0.0694089)  
( 40,0.0423453)  
};    
\addplot[
    color=green,
    mark=triangle,
    ]
    coordinates {
( 7,0.0002393877747)  
( 10,3.592079105e-005)  
( 14,2.1481843e-006)  
( 20,6.166071698e-008)  
( 28,6.608705527e-010)  
( 40,1.32938105e-012)  
};
%\legend{SGD, N.U. SGD,N.U. SAG,N.U. SVRG, LSVRG, Q-SVRG}
\end{axis}
\end{tikzpicture}
\end{subfigure}
\comment{
\caption{Convergence on the sonar, madelon, and sido0 datasets with \(\lambda=\bar L/n\)}
\label{fig:sido0Lambda1}
\end{figure}

\begin{figure}
\centering
}

\begin{subfigure}[b]{0.3\textwidth}
\centering
\begin{tikzpicture}[scale=.45]
%\begin{tiny}
\begin{axis}[
    title={},
    ymode=log,
    xlabel={Number of effective passes},
    ylabel={$g(\theta)-g(\theta^*)$},
    legend pos=south west,
    ymajorgrids=true,
    grid style=dashed,
]
 \addplot[
    color=magenta,
    mark=o,
    ]
    coordinates {
( 13,0.0404091)  
( 19,0.0336552)  
( 27,0.0275294)  
( 39,0.02207)  
( 55,0.016672)  
( 78,0.0120018)  
( 111,0.00805407)  
( 157,0.00512939)  
    };
 \addplot[
    color=red,
    mark=diamond,
    ]
    coordinates {
( 13,0.00861098)  
( 19,0.00549535)  
( 27,0.00381943)  
( 39,0.00314718)  
( 55,0.00181371)  
( 78,0.00112818)  
( 111,0.00079532)  
( 157,0.000716989)  
};
 \addplot[
    color=darkgray,
    mark=star,
    ]
    coordinates {
 ( 13,0.00510667)  
( 19,0.00275855)  
( 27,0.00150442)  
( 39,0.000172737)  
( 55,3.3758e-005)  
( 78,7.96563e-006)  
( 111,7.44803e-009)  
( 157,3.37609e-011)  
   };
 \addplot[
    color=blue,
    mark=square,
    ]
    coordinates {
( 12,0.02502540758)  
( 18,0.01834541245)  
( 27,0.01244206862)  
( 39,0.007971931433)  
( 54,0.004881686401)  
( 78,0.002428969106)  
( 111,0.001029869503)  
( 156,0.0003541836451)  
}; 
\addplot[
    color=orange,
    mark=otimes,
    ]
    coordinates {
( 12,0.0740708)  
( 16,0.0617902)  
( 24,0.0512965)  
( 32,0.0425004)  
( 57,0.0349779)  
( 81,0.0283381)  
( 115,0.0223627)  
( 154,0.017002)  
};    
\addplot[
    color=green,
    mark=triangle,
    ]
    coordinates {
( 13,0.003187335863)  
( 17,0.00151644883)  
( 23,0.0006153974132)  
( 31,0.0002186736414)  
( 43,5.679710402e-005)  
( 60,8.276380106e-006)  
( 85,4.537086229e-007)  
( 122,7.404721031e-009)  
( 172,4.331721093e-011)  
};
%\legend{SGD, N.U. SGD,N.U. SAG,N.U. SVRG, LSVRG, Q-SVRG}
\end{axis}
\end{tikzpicture}
%\caption{Convergence on sonar dataset with \(\lambda=0.1\bar L/n\)}
\label{fig:sonarlambda0.1}
\end{subfigure}
\begin{subfigure}[b]{0.3\textwidth}
\centering
\begin{tikzpicture}[scale=.45]
\begin{axis}[
    title={},
    ymode=log,
    xlabel={Number of effective passes},
    ylabel={$g(\theta)-g(\theta^*)$},
    legend pos=south west,
    ymajorgrids=true,
    grid style=dashed,
]
 \addplot[
    color=magenta,
    mark=o,
    ]
    coordinates {
( 11,0.00932233)  
( 16,0.00627629)  
( 23,0.00453706)  
( 32,0.00338732)  
( 46,0.00226462)  
( 65,0.00148474)  
( 92,0.00099505)  
( 131,0.00064056)  
( 185,0.000476314)  
    };
 \addplot[
    color=red,
    mark=diamond,
    ]
    coordinates {
( 11,0.0122424)  
( 16,0.00942498)  
( 23,0.00629877)  
( 32,0.00461169)  
( 46,0.00317798)  
( 65,0.00214634)  
( 92,0.00147509)  
( 131,0.0011313)  
( 185,0.000737162)  
};
 \addplot[
    color=darkgray,
    mark=star,
    ]
    coordinates {
( 11,0.00499746)  
( 16,0.00306265)  
( 23,0.00183213)  
( 32,0.000974988)  
( 46,0.00050881)  
( 65,0.000257932)  
( 92,8.68904e-005)  
( 131,6.02113e-006)  
( 185,1.82143e-007)  
   };
 \addplot[
    color=blue,
    mark=square,
    ]
    coordinates {
( 9,0.004793157749)  
( 15,0.001520391225)  
( 21,0.0007797593104)  
( 30,0.0004905931347)  
( 45,0.0003471043954)  
( 63,0.0002512854024)  
( 90,0.0001564418065)  
( 129,7.897158099e-005)  
( 183,3.066502554e-005)  
}; 
\addplot[
    color=orange,
    mark=otimes,
    ]
    coordinates {
( 11,0.017044)  
( 16,0.0106427)  
( 24,0.00631605)  
( 31,0.00363966)  
( 44,0.00209828)  
( 60,0.00124252)  
( 88,0.000761036)  
( 119,0.000473574)  
( 171,0.000290071)  
};    
\addplot[
    color=green,
    mark=triangle,
    ]
    coordinates {
( 12,0.0003069120081)  
( 15,0.0001892667997)  
( 20,0.0001022443154)  
( 27,4.672786875e-005)  
( 36,1.591697662e-005)  
( 50,4.062080863e-006)  
( 71,4.017717182e-007)  
( 101,1.678073569e-008)  
( 144,1.940089756e-010)  
( 203,3.862465903e-013)  
};
%\legend{SGD, N.U. SGD,N.U. SAG,N.U. SVRG, LSVRG, Q-SVRG}
\end{axis}
\end{tikzpicture}
%\caption{Convergence on madelon dataset with \(\lambda=0.1\bar L/n\)}
\label{fig:madelonlambda0.1}
\end{subfigure}
\begin{subfigure}[b]{0.3\textwidth}
\centering
\begin{tikzpicture}[scale=.45]
\begin{axis}[
    title={},
    ymode=log,
    xlabel={Number of effective passes},
    ylabel={$g(\theta)-g(\theta^*)$},
    legend pos=south west,
    ymajorgrids=true,
    grid style=dashed,
]
 \addplot[
    color=magenta,
    mark=o,
    ]
    coordinates {
( 10,0.1279)  
( 14,0.0847125)  
( 20,0.0514778)  
( 29,0.0291881)  
( 41,0.0159421)  
( 58,0.00864008)  
( 82,0.00468372)  
( 116,0.00252544)  
    };
 \addplot[
    color=red,
    mark=diamond,
    ]
    coordinates {
( 10,0.00248747)  
( 14,0.00155365)  
( 20,0.000975)  
( 29,0.000597016)  
( 41,0.000389532)  
( 58,0.000255695)  
( 82,0.000166632)  
( 116,0.000106088)  
};
 \addplot[
    color=darkgray,
    mark=star,
    ]
    coordinates {
 ( 10,0.00757938)  
( 14,0.00480397)  
( 20,0.00286306)  
( 29,0.00158569)  
( 41,0.000868751)  
( 58,0.000444918)  
( 82,0.000226335)  
( 116,9.09443e-005)  
  };
 \addplot[
    color=blue,
    mark=square,
    ]
    coordinates {
( 12,0.009359758946)  
( 18,0.002893060283)  
( 27,0.001316588956)  
( 39,0.0007300019463)  
( 57,0.0003645470935)  
( 81,0.0001697840999)  
( 114,6.996795841e-005)  
}; 
\addplot[
    color=orange,
    mark=otimes,
    ]
    coordinates {
( 12,0.232123)  
( 18,0.183153)  
( 27,0.134049)  
( 40,0.089985)  
( 55,0.0553122)  
( 82,0.0315897)  
( 123,0.0172637)  
};    
\addplot[
    color=green,
    mark=triangle,
    ]
    coordinates {
( 11,0.0009005821751)  
( 14,0.0005117661585)  
( 18,0.0002575314858)  
( 24,0.0001129713621)  
( 33,4.273391848e-005)  
( 45,1.346452514e-005)  
( 63,1.558959327e-006)  
( 90,6.0812906e-008)  
( 127,1.48928346e-009)  
};
%\legend{SGD, N.U. SGD,N.U. SAG,N.U. SVRG, LSVRG, Q-SVRG}
\end{axis}
\end{tikzpicture}
\end{subfigure}
\comment{
\caption{Convergence on the sonar, madelon and sido0 datasets with \(\lambda=0.1\bar L/n\)}
\label{fig:sido0lambda0.1}
\end{figure}

\begin{figure}
\centering
}
\begin{subfigure}[b]{0.3\textwidth}
\centering
\begin{tikzpicture}[scale=.45]
%\begin{tiny}
\begin{axis}[
    title={},
    ymode=log,
    xlabel={Number of effective passes},
    ylabel={$g(\theta)-g(\theta^*)$},
%    legend pos=south west,
    legend style={at={(0.5,-0.25)},anchor=north},
    ymajorgrids=true,
    grid style=dashed,
]
 \addplot[
    color=magenta,
    mark=o,
    ]
    coordinates {
( 9,0.0714197)  
( 13,0.0624177)  
( 19,0.0550647)  
( 27,0.0481724)  
( 39,0.0417196)  
( 55,0.0349955)  
( 78,0.0285829)  
( 111,0.0224899)  
( 157,0.017133)  
    };
%  \addlegendentry{$xGFGFFFJFF$} 
  \addplot[
    color=red,
    mark=diamond,
    ]
    coordinates {
( 9,0.0268667)  
( 13,0.0178546)  
( 19,0.0123828)  
( 27,0.00886466)  
( 39,0.00638744)  
( 55,0.0035816)  
( 78,0.00220501)  
( 111,0.00132075)  
( 157,0.000960793)  
};
 \addplot[
    color=darkgray,
    mark=star,
    ]
    coordinates {
( 9,0.021509)  
( 13,0.0151268)  
( 19,0.010006)  
( 27,0.00646353)  
( 39,0.00129085)  
( 55,0.000466948)  
( 78,0.000376045)  
( 111,4.43786e-006)  
( 157,6.78413e-008)  
   };
 \addplot[
    color=blue,
    mark=square,
    ]
    coordinates {
( 9,0.05149973497)  
( 12,0.04577362632)  
( 18,0.03782358465)  
( 27,0.03014221555)  
( 39,0.02355171212)  
( 54,0.01820205287)  
( 78,0.01288089263)  
( 111,0.008645973558)  
( 156,0.005424923038)  
}; 
\addplot[
    color=orange,
    mark=otimes,
    ]
    coordinates {
( 10,0.11224)  
( 12,0.0974441)  
( 16,0.0848432)  
( 24,0.0739623)  
( 32,0.0647019)  
( 57,0.0566015)  
( 81,0.04921)  
( 115,0.0422366)  
( 154,0.0355605)  
};    
\addplot[
    color=green,
    mark=triangle,
    ]
    coordinates {
( 10,0.01992531389)  
( 13,0.01433603052)  
( 17,0.009942603011)  
( 23,0.006433602116)  
( 31,0.00394866649)  
( 43,0.002127585079)  
( 59,0.001007008911)  
( 82,0.0004342880649)  
( 115,0.0001537686779)  
( 161,3.896356756e-005)  
};
\legend{SGD, N.U. SGD,N.U. SAG,N.U. SVRG, LSVRG, Q-SVRG}
\end{axis}
\end{tikzpicture}
%\caption{Convergence on sonar dataset with \(\lambda=0.01\bar L/n\)}
\label{fig:sonarlambda0.01}
\end{subfigure}
\begin{subfigure}[b]{0.3\textwidth}
\centering
\begin{tikzpicture}[scale=.45]
\begin{axis}[
    title={},
    ymode=log,
    xlabel={Number of effective passes},
    ylabel={$g(\theta)-g(\theta^*)$},
%    legend pos=south west,
    legend style={at={(0.5,-0.25)},anchor=north},
    ymajorgrids=true,
    grid style=dashed,
]
 \addplot[
    color=magenta,
    mark=o,
    ]
    coordinates {
( 11,0.0114427)  
( 16,0.00821129)  
( 23,0.0063434)  
( 32,0.00507307)  
( 46,0.0038109)  
( 65,0.00285997)  
( 92,0.00217214)  
( 131,0.00157912)  
( 185,0.00117078)  
    };
 \addplot[
    color=red,
    mark=diamond,
    ]
    coordinates {
( 11,0.0142235)  
( 16,0.011085)  
( 23,0.00756822)  
( 32,0.00556678)  
( 46,0.00382469)  
( 65,0.00256304)  
( 92,0.00174068)  
( 131,0.00132024)  
( 185,0.000876803)  
};
 \addplot[
    color=darkgray,
    mark=star,
    ]
    coordinates {
( 11,0.00803928)  
( 16,0.00574138)  
( 23,0.00409661)  
( 32,0.00261433)  
( 46,0.00160342)  
( 65,0.000973203)  
( 92,0.000529954)  
( 131,0.000282931)  
( 185,0.00014678)  
   };
 \addplot[
    color=blue,
    mark=square,
    ]
    coordinates {
( 9,0.007138087332)  
( 15,0.003406958812)  
( 21,0.002497053822)  
( 30,0.002109702905)  
( 45,0.001887309449)  
( 63,0.001704781475)  
( 90,0.001468925795)  
( 129,0.001185231066)  
( 183,0.000881036949)  
}; 
\addplot[
    color=orange,
    mark=otimes,
    ]
    coordinates {
( 11,0.020105)  
( 16,0.0132992)  
( 24,0.00861911)  
( 31,0.00566747)  
( 44,0.00393096)  
( 60,0.00293873)  
( 88,0.00234646)  
( 119,0.00194187)  
( 171,0.00161095)  
};    
\addplot[
    color=green,
    mark=triangle,
    ]
    coordinates {
( 9,0.002037124911)  
( 12,0.001764226297)  
( 15,0.001507426128)  
( 20,0.001230884637)  
( 27,0.0009387161911)  
( 36,0.0006423576342)  
( 50,0.0003832722543)  
( 69,0.000185801979)  
( 96,7.279842128e-005)  
( 135,2.108737739e-005)  
( 189,4.56239749e-006)  
};
\legend{SGD, N.U. SGD,N.U. SAG,N.U. SVRG, LSVRG, Q-SVRG}
\end{axis}
\end{tikzpicture}
%\caption{Convergence on madelon dataset with \(\lambda=0.01\bar L/n\)}
\label{fig:madelonlambda0.01}
\end{subfigure}
\begin{subfigure}[b]{0.3\textwidth}
\centering
\begin{tikzpicture}[scale=.45]
\begin{axis}[
    title={},
    ymode=log,
    xlabel={Number of effective passes},
    ylabel={$g(\theta)-g(\theta^*)$},
%    legend pos=south west,
    legend style={at={(0.5,-0.25)},anchor=north},
    ymajorgrids=true,
    grid style=dashed,
]
 \addplot[
    color=magenta,
    mark=o,
    ]
    coordinates {
 ( 10,0.138913)  
( 14,0.0933848)  
( 20,0.0578378)  
( 29,0.0336378)  
( 41,0.0190468)  
( 58,0.0108812)  
( 82,0.006359)  
( 116,0.00379354)  
( 165,0.00229032)  
   };
 \addplot[
    color=red,
    mark=diamond,
    ]
    coordinates {
( 10,0.00356923)  
( 14,0.00238622)  
( 20,0.00159734)  
( 29,0.00104569)  
( 41,0.000699083)  
( 58,0.000459793)  
( 82,0.000295014)  
( 116,0.000183266)  
( 165,0.000108813)  
};
 \addplot[
    color=darkgray,
    mark=star,
    ]
    coordinates {
( 10,0.0110725)  
( 14,0.00790038)  
( 20,0.00539303)  
( 29,0.00347982)  
( 41,0.00235608)  
( 58,0.0014495)  
( 82,0.0010533)  
( 116,0.000818821)  
( 165,0.000487859)  
  };
 \addplot[
    color=blue,
    mark=square,
    ]
    coordinates {
( 9,0.02658663204)  
( 12,0.01261798448)  
( 18,0.004972072173)  
( 27,0.002942598191)  
( 39,0.002082758456)  
( 57,0.001444019259)  
( 81,0.001008060021)  
( 114,0.0006961510209)  
( 165,0.0004521957907)  
}; 
\addplot[
    color=orange,
    mark=otimes,
    ]
    coordinates {
( 9,0.292382)  
( 12,0.246732)  
( 18,0.196341)  
( 27,0.145304)  
( 40,0.098932)  
( 55,0.0619119)  
( 82,0.0361902)  
( 123,0.0204227)  
( 173,0.0115184)  
};    
\addplot[
    color=green,
    mark=triangle,
    ]
    coordinates {
( 11,0.002186247558)  
( 14,0.001559063629)  
( 18,0.001075399698)  
( 24,0.0007157473094)  
( 33,0.0004619081427)  
( 45,0.0002841979686)  
( 62,0.0001662979011)  
( 86,8.760390549e-005)  
( 120,4.265508207e-005)  
( 169,1.777417363e-005)  
};
\legend{SGD, N.U. SGD,N.U. SAG,N.U. SVRG, LSVRG, Q-SVRG}
\end{axis}
\end{tikzpicture}
\end{subfigure}
\caption{Convergence on the sonar, madelon and sido0 datasets, from left to right. The first, second and third row correspond respectively to \(\lambda=\bar L/n\), \(\lambda=0.1\bar L/n\) and \(\lambda=0.01\bar L/n\).}
\label{fig:sido0lambda0.01}
\end{figure}
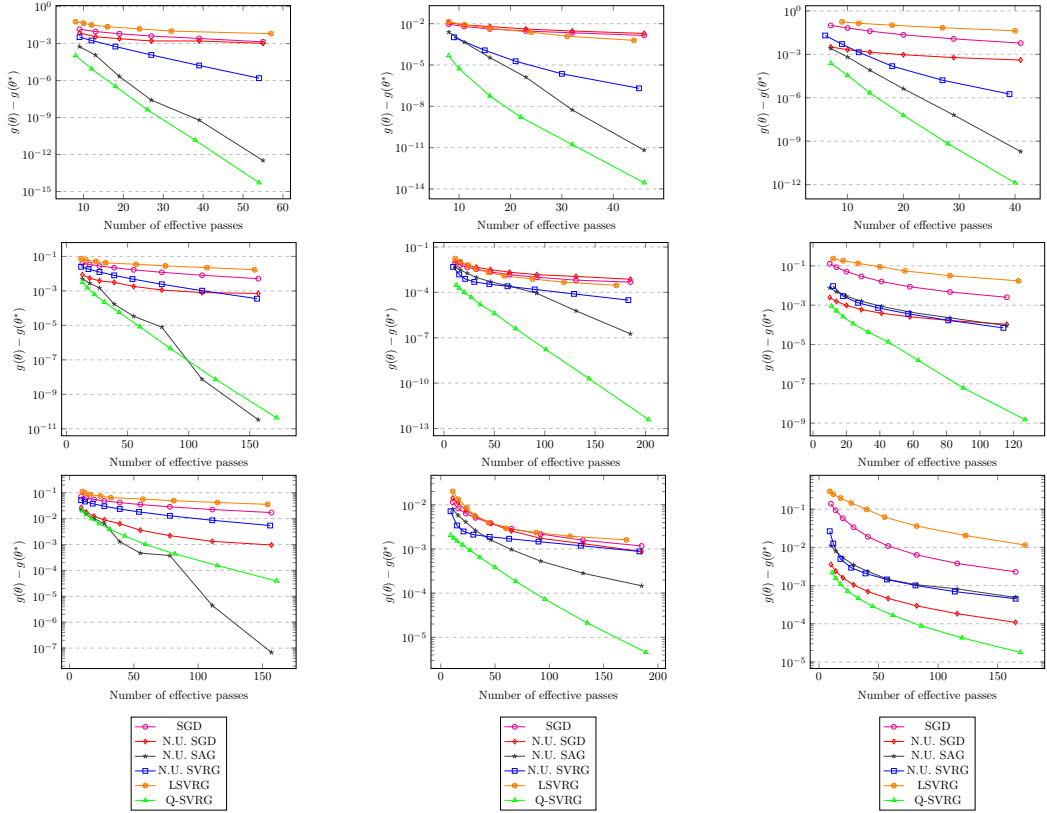

\section*{Conclusion}\label{se:conclusion}
We have analysed a variant of SVRG to approximately minimize a class of quadratic functions.  Our method, Q-SVRG,  is applicable to  minimization problems  involving \(n\) points in dimension \(d\) that arise in several applications such as least-squares regressions, ridge regressions, linear discriminant analysis and regularized linear discriminant analysis. When the Hessian matrix is  positive definite, Q-SVRG yields a convergence rate of  \(O((\kappa /m)^{l})\) in \(O(l(n+m))\) stochastic gradients for  arbitrary \(l\ge1\), and can be simulated  without the knowledge of \(\mu\).  In addition, when \(\mu\) is known, Q-SVRG achieves a linear convergence rate and improves the previously known running time of SVRG   by up a logarithmic factor. Furthermore, when \(\kappa\) is a constant, our analysis improves the state-of-the-art running times of accelerated SVRG,  and is better than the matching lower bound of \citet{LanZhouAccel2018}, by a logarithmic factor. A limitation of Q-SVRG  is that it is applicable only to quadratic problems.

\section{Acknowledgements}
This work was achieved through the Laboratory
of Excellence on Financial Regulation (Labex ReFi) under the reference ANR-10-LABX-0095.
\appendix
\section{Proof of Theorem \ref{th:mainGen}}\label{se:proofMain}
We prove the theorem in the case \(l=1\). An inductive argument then implies that Theorem \ref{th:mainGen} holds for any \(l\geq1\).  We first show that it can be assumed without loss of generality that \(\theta_{0}=0\).  Let \(\hat \theta^*:=\theta^{*}-\theta_{0}\), \(\hat c:=H\hat \theta^*\),   \(\hat\theta_{k}:=\theta_{k}- \theta_{0}\) for \(k\geq0\), and\begin{displaymath}
\hat f(\theta):=\frac{1}{2}\theta^{T}H\theta-\hat c^{T}\theta,
\end{displaymath}    for any \(d\)-dimensional column vector \(\theta\).   Then \eqref{eq:Htheta*} holds for the triplet \((H,\hat \theta^*,\hat c)\), and the sequence   \((\hat\theta_{k}:k\geq0)\),  satisfies the recursion
\begin{equation}
\hat\theta_{k+1}=\hat\theta_{k}-\alpha(Q_{k}\hat\theta_{k}-\hat c),
\end{equation}  
which is of the  same type as \eqref{eq:BasicDefthetaGen}. For \(k\geq1\), let \begin{equation*}
\hat{\bar\theta}_{k}:=\frac{\hat\theta_{0}+\cdots+\hat\theta_{k-1}}{k}.
\end{equation*}
Applying~\eqref{eq:minimumf} to the triplet \((H,\hat \theta^*,\hat c)\) shows that, for any \(d\)-dimensional column vector \(\theta\),
 \begin{displaymath}
f(\theta)-f(\theta^{*})=\hat f(\theta-\theta_{0})-\hat f(\hat \theta^*).
\end{displaymath}   Consequently, as  \(\hat{\bar\theta}_{k}=\bar\theta_{k}- \theta_{0}\), \begin{displaymath}
f(\bar\theta_{k})-f(\theta^{*})=\hat f(\hat{\bar\theta}_{k})-\hat f(\hat \theta^*).
\end{displaymath}Thus, if  Theorem~\ref{th:mainGen} holds for the triplet \((H,\hat \theta^*,\hat c)\) and the sequence   \((\hat\theta_{k}:k\geq0)\),
which satisfies \(\hat\theta_{0}=0\), it also holds for the triplet \((H,\theta^*,c)\) and the sequence   \((\theta_{k}:k\geq0)\). The rest of the proof assumes that \(\theta_0=0\). We will also assume without loss of generality that \(L=1\). This assumption can be justified by a suitable scaling of \(\alpha\), \(H\) and \((Q_{k}:k\ge0)\).

For \(k\geq0\), define  the \(d\times d\) random matrix \(P_{k}:=I-\alpha Q_{k}\).   Thus \eqref{eq:BasicDefthetaGen} can be rewritten as   \begin{equation}\label{eq:deftheta}
\theta_{k+1}=P_{k}\theta_{k}+\alpha c.
\end{equation} 
  By A2, we have\begin{equation}
\label{eq:EPk}
E(P_{k})=I-\alpha H.
\end{equation} 
Define the sequence of  \(d\)-dimensional column vectors
\((\beta_k:k\ge0)\) recursively as follows. Let \(\beta_0 = \theta^{*}\) and, for \(k\ge0\), let \begin{equation}\label{eq:defbeta}
\beta_{k+1}=P_{k}\beta_{k}+\alpha c.
\end{equation}
Thus,  
\((\beta_k:k\ge0)\) satisfies the same recursion as   
\((\theta_k:k\ge0)\). It follows by induction from~\eqref{eq:deftheta} and~\eqref{eq:defbeta} that, for any \(k\ge0\), the vectors \(\theta_k\) and \(\beta_k\) are square-integrable. \subsection{Bounding the bias}
\begin{lemma}
For \(k\geq 0\), we have
\begin{equation}
\label{eq:Ethetak}E(\theta_{k})=(I-(I-\alpha H)^{k})\theta^{*}
\end{equation}
and
\begin{equation}
\label{eq:Ebetak}E(\beta_{k})=\theta^{*}.
\end{equation}
 
\end{lemma}
\begin{proof}
By~\eqref{eq:deftheta}, \(\theta_{k}\) is  a deterministic function of \(P_0,\cdots,P_{k-1}\), and so  \(\theta_{k}\) is independent of \(P_{k}\). We prove~\eqref{eq:Ethetak} by induction on \(k\). Clearly, \eqref{eq:Ethetak} holds for \(k=0\). Assume that  \eqref{eq:Ethetak}  holds for \(k\). Thus, 
\begin{eqnarray*}
E(\theta_{k+1})
&=&E(P_{k})E(\theta_{k})+\alpha c\\
&=&(I-\alpha H)(I-(I-\alpha H)^{k})\theta^{*}+\alpha H\theta^{*}\\
&=&(I-(I-\alpha H)^{k+1})\theta^{*},
\end{eqnarray*}and so \eqref{eq:Ethetak}  holds for \(k+1\).    
A similar inductive proof implies~\eqref{eq:Ebetak}.\end{proof} 
\begin{lemma}\label{le:bias}
For \(k\geq1\), \begin{equation*}
E((\bar\theta_{k}-\theta^{*})^{T})\,H\, E(\bar\theta_{k}-\theta^{*})\leq\frac{||\theta^{*}||^{2}}{\alpha k}.
\end{equation*}\end{lemma}
\begin{proof}By~\eqref{eq:Ethetak}, for \(0\leq i\leq k-1\), 
\begin{equation*}E((\theta_{i}-\theta^{*})^{T})\,H\, E(\theta_{i}-\theta^{*})
={\theta^{*}}^{T}(I-\alpha H)^{2i}H\theta^{*}.\end{equation*}
As \(H\) is symmetric positive semidefinite, the quadratic function \(x\mapsto x^{T}Hx\) is convex over \(\mathbb{R}^{d}\), and so \begin{eqnarray*}E((\bar\theta_{k}-\theta^{*})^{T})\,H\, E(\bar\theta_{k}-\theta^{*})
&\le&\frac{1}{k}\sum^{k-1}_{i=0}E((\theta_{i}-\theta^{*})^{T})\,H\, E(\theta_{i}-\theta^{*})\\
&=&\frac{1}{k}\sum^{k-1}_{i=0}{\theta^{*}}^{T}(I-\alpha H)^{2i}H\theta^{*}\\
&\le&\frac{1}{k}\sum^{2k-2}_{i=0}{\theta^{*}}^{T}(I-\alpha H)^{i}H\theta^{*}\\
&=&\frac{1}{\alpha k}{\theta^{*}}^{T}(I-(I-\alpha H)^{2k-1})\theta^{*}.
\end{eqnarray*}
The third equation follows by observing that \((I-\alpha H)^{i}H\) is positive semidefinite since all eigenvalues of \(\alpha H\) are between \(0\) and \(1\).
The last equation follows from the identity
\begin{equation}\label{eq:identitySumH}
\sum^{j}_{i=0}(I-H')^{i}H'=I-(I-H')^{j+1}, \ j\ge0,
\end{equation} 
for any \(d\times d\) matrix \(H'\). As \((I-\alpha H)^{2k-1}\) is positive semidefinite, this completes the proof.\end{proof}
\subsection{Bounding the variance}      
\begin{lemma}
\label{le:VarianceBetaKPositiveDefinite}
For \(k\ge0\), we have \(E(||\beta_{k}-\theta^{*}||{^2})\le(\alpha/\mu){\theta^{*}}^{T}H\theta^{*}\).
\end{lemma}
\begin{proof}
By~\eqref{eq:Htheta*} and~\eqref{eq:defbeta}, we have \(\beta_{k}=P_{k-1}\beta_{k-1}+\alpha H\theta^{*}\). Hence\begin{equation}\label{eq:betakMinustheta*}
\beta_{k}-\theta^{*}=P_{k-1}(\beta_{k-1}-\theta^{*})+\alpha(H-Q_{k-1})\theta^{*}. \end{equation} Since   \(\beta_{k-1}\) and \(Q_{k-1}\) are independent, it follows from \eqref{eq:Ebetak} that \begin{equation*}E((\beta_{k-1}-\theta^{*})^{T}P_{k-1}(H-Q_{k-1})\theta^{*})=0.
\end{equation*}As \(E(||U+V||^{2})=E(||U||^{2})+E(||V||^{2})\) for any square-integrable random \(d\)-dimensional column vectors \(U\) and \(V\) with \(E(U^{T}V)=0\), \eqref{eq:betakMinustheta*} implies that\begin{equation}\label{betakDiffExpanded}
E(||\beta_{k}-\theta^{*}||^{2})=E((\beta_{k-1}-\theta^{*})^{T}{P_{k-1}}^{2}(\beta_{k-1}-\theta^{*}))+\alpha^{2}{\theta^{*}}^{T}{E((H-Q_{k-1})^{2})\theta^{*}}.
\end{equation}
On the other hand, \begin{eqnarray*}E({P_{j}}^{2})&=&E(I - 2\alpha Q_{j}+\alpha^{2}{Q_{j}}^{2})\\
&\le&I-2\alpha H+\alpha^{2}H\\
&\le&I-\alpha H\\
&\le&(1-\alpha\mu)I,
\end{eqnarray*} and \begin{eqnarray*}E((H-Q_{j})^{2})
&=&E(H^{2}-HQ_{j}-Q_{j}H+{Q_{j}}^{2})\\
&\le&E({Q_{j}}^{2})-H^{2}\\
&\leq&H.
\end{eqnarray*}
As \(P_{k-1}\) and \(\beta_{k-1}\) are independent, it follows from~\eqref{betakDiffExpanded} that \begin{equation}\label{betakDiffExpanded}
E(||\beta_{k}-\theta^{*}||^{2})\le (1-\alpha\mu)E(||\beta_{k-1}-\theta^{*}||^{2})+\alpha^{2}{\theta^{*}}^{T}{H\theta^{*}}.
\end{equation}
An induction on \(k\) completes the proof.
\end{proof}

For any square-integrable  \(d\)-dimensional random column vectors \(U\) and \(V\), let \begin{equation}\label{eq:covDefinition}
\cov(U,V):=E(U^{T}V)-E(U^{T})E(V),
\end{equation}and \(\var(U):=\cov(U,U)\).
For any square-integrable  \(d\)-dimensional random column vectors \(U\),  \(V\) and \(V'\), any deterministic symmetric  \(d\times d\) matrix \(A\), and any bounded  \(d\times d\) random matrix  \(B\)  independent of \((U,V)\),  it can be shown that \(\cov(U,V+V')=\cov(U,V)+\cov(U,V')\), and   \(\cov(AU,V)=\cov(U,AV)\), with \(\cov(U,BV)=\cov(U,E(B)V)\) and \(\var(U+V)\le2(\var(U)+\var(V))\). Furthermore, if \(A\)    is positive semidefinite, then \(\cov(U,AU)\ge0\). 

For \(0\leq j\leq  k\), let  \(M_{j,k}=P_{{k-1}}P_{{k-2}}\cdots P_{{j}}\), with \(M_{k,k}=I\).    
\begin{lemma}
\label{le:covHthetathetakpj}For nonnegative integers \(k,j\), we have
\begin{equation*}
\cov(H\theta_{k},\theta_{k+j})=\cov(H\theta_{k},(I-\alpha H)^{j}\theta_{k}).
\end{equation*} 
\end{lemma}
\begin{proof}
We show by induction on \(j\) that, for \(j\ge0\), \begin{equation}
\label{eq:rec}\theta_{k+j}=M_{k,k+j}\theta_{k}+\alpha\sum^{k+j}_{i=k+1}M_{i,k+j}c.
\end{equation}Clearly, \eqref{eq:rec} holds for \(j=0\). Assume now that  \eqref{eq:rec} holds for \(j\). Then\begin{eqnarray*}
\theta_{k+j+1}&=&
P_{k+j}\theta_{k+j}+\alpha c\\
&=&P_{k+j}(M_{k,k+j}\theta_{k}+\alpha\sum^{k+j}_{i=k+1}M_{i,k+j}c)+\alpha c\\
&=&M_{k,k+j+1}\theta_{k}+\alpha\sum^{k+j}_{i=k+1}M_{i,k+j+1}c+c,
\end{eqnarray*}
and so   \eqref{eq:rec} holds for \(j\)+1. As \(\theta_{k}\) is independent of \(M_{i,k+j}\), for \(k\leq i\leq k+j\), it follows from  \eqref{eq:rec} that\begin{eqnarray*}
\cov(H\theta_{k},\theta_{k+j})&=&
\cov(H\theta_{k},M_{k,k+j}\theta_{k})\\
&=&\cov(H\theta_{k},E(M_{k,k+j})\theta_{k})\\
&=&\cov(H\theta_{k},(I-\alpha H)^{j}\theta_{k}).
\end{eqnarray*}The last equation follows from~\eqref{eq:EPk}.
\end{proof}
\begin{lemma}\label{le:covSum} For \(0\leq i\leq k\), we have
\begin{displaymath}
\sum^{k}_{j=i}\cov(H\theta_{i},\theta_{j})\le \frac{4}{\alpha\mu}{\theta^{*}}^{T}H\theta^{*}.
\end{displaymath}
\end{lemma}
\begin{proof}
By Lemma~\ref{le:covHthetathetakpj},
\begin{eqnarray*}
\sum^{k}_{j=i}\cov(H\theta_{i},\theta_{j})
&=&\sum^{k}_{j=i}\cov(H\theta_{i},(I-\alpha H)^{j-i}\theta_{i})\\
&=&\sum^{k}_{j=i}\cov(\theta_{i},H(I-\alpha H)^{j-i}\theta_{i})\\
&=&\alpha^{-1}\cov(\theta_{i},(I-(I-\alpha H)^{k+1-i})\theta_{i})\\
&\le&\alpha^{-1}\var(\theta_{i}).
\end{eqnarray*}
The third equation follows from~\eqref{eq:identitySumH}, and the last one from  the  positive semidefiniteness of \((I-\alpha H)^{k+1-i}\). On the other hand,
\begin{eqnarray*}
\var(\theta_{i}) 
&\le& 2(\var(\beta_{i})+\var(\theta_{i}-\beta_{i})) \\
&\le& 2(E(||\beta_{i}-\theta^{*}||{^2})+||\theta^{*}||^{2}) \\
&\le& \frac{4}{\mu}{\theta^{*}}^{T}H\theta^{*}.
\end{eqnarray*}
The second equation follows from the inequality \(E(||\theta_{i}-\beta_{i}||^{2})\leq||\theta^{*}||^{2}\),  which can be shown by induction on \(i\). The last equation is a consequence of Lemma~\ref{le:VarianceBetaKPositiveDefinite} and the inequality \({\theta^{*}}^{T}H\theta^{*}\ge\mu||{\theta^{*}}||^{2}\). This concludes the proof.
\end{proof}
\subsection{Combining bias and variance terms} 
By~\eqref{eq:minimumf},
\begin{displaymath}
E(f(\bar\theta_{k}))-f(\theta^{*})=\frac{1}{2}\,E((\bar\theta_{k}-\theta^{*})^{T}\,H\,(\bar\theta_{k}-\theta^{*})).
\end{displaymath}
On the other hand, by~\eqref{eq:covDefinition}, 
\begin{eqnarray*}
E((\bar\theta_{k}-\theta^{*})^{T}\,H\,(\bar\theta_{k}-\theta^{*}))
&=&E((\bar\theta_{k}-\theta^{*})^{T})\,H\,E(\bar\theta_{k}-\theta^{*})+\cov(H(\bar\theta_{k}-\theta^{*}),\bar\theta_{k}-\theta^{*})\\
&=&E((\bar\theta_{k}-\theta^{*})^{T})\,H\,E(\bar\theta_{k}-\theta^{*})+\cov(H\bar\theta_{k},\bar\theta_{k}).
\end{eqnarray*}Because \(H\) is symmetric positive semidefinite,\begin{eqnarray*}
\cov(H\bar\theta_{k},\bar\theta_{k})&=&\frac{1}{k^{2}}(\sum^{k-1}_{i=0}\cov(H\theta_{i},\theta_{i})+2\sum^{k-1}_{i=0}\sum^{k-1}_{j=i+1}\cov(H\theta_{i},\theta_{j}))
\\&\le&\frac{2}{k^{2}}\sum^{k-1}_{i=0}\sum^{k-1}_{j=i}\cov(H\theta_{i},\theta_{j})
\\&\le&\frac{8}{\alpha\mu k}{\theta^{*}}^{T}H\theta^{*}.
\end{eqnarray*}The last equation follows from Lemma~\ref{le:covSum}. Together with Lemma~\ref{le:bias}, this implies that
\begin{eqnarray*}
E((\bar\theta_{k}-\theta^{*})^{T}\,H\,(\bar\theta_{k}-\theta^{*}))&\leq& \frac{||\theta^{*}||^{2}}{\alpha k}+\frac{8}{\alpha\mu k}{\theta^{*}}^{T}H\theta^{*}\\
&\le&\frac{9}{\alpha\mu k}{\theta^{*}}^{T}H\theta^{*},
\end{eqnarray*}
where the second equation follows from the inequality \({\theta^{*}}^{T}H\theta^{*}\ge\mu||{\theta^{*}}||^{2}\).
It follows that
\begin{eqnarray*}
E(f(\bar\theta_{k}))-f(\theta^{*})&\leq&\frac{9{\theta^{*}}^{T}H\theta^{*}}{2\alpha\mu k}\\
&=&\frac{9(f(0)-f(\theta^{*}))}{\alpha\mu k}.
\end{eqnarray*}
This concludes the proof.
\bibliography{poly}
\end{document}